\documentclass[letterpaper, 10 pt, conference]{ieeeconf}  

\IEEEoverridecommandlockouts                              

\usepackage{graphicx}      

\usepackage{amsmath,graphicx}
\usepackage{physics}
\usepackage{amsfonts}
\usepackage{comment}
\usepackage{url}
\usepackage{soul}

\newcommand{\llrrparen}[1]{
  \left(\mkern-6mu\left(#1\right)\mkern-6mu\right)}
\newcommand{\psymbol}[2]{\genfrac{}{}{0pt}{}{#1}{#2}}

\let\labelindent\relax
\usepackage{enumitem}

\usepackage{nccmath} 

\usepackage{balance}

\usepackage{dsfont} 

\usepackage{tikz}
\usepackage{pgfplots}
    \usepgfplotslibrary{groupplots}
\tikzset{every picture/.style={line width=0.75pt}} 
\pgfplotsset{compat=newest}
\usepgfplotslibrary{patchplots}
\usepackage{pgfplotstable}

\pgfplotsset{width=7cm,compat=1.13}
\usepgfplotslibrary{statistics}
  
\usepgfplotslibrary{fillbetween}
\usepackage{filecontents}
\definecolor{blue_c}{RGB}{0, 122, 204}

\newtheorem{theorem}{Theorem}
\newtheorem{lemma}{Lemma}

\newtheorem{remark}{Remark}

\pgfdeclareplotmark{x}
{%
  \pgfpathmoveto{\pgfqpoint{-.70710678\pgfplotmarksize}{-.70710678\pgfplotmarksize}}%
  \pgfpathlineto{\pgfqpoint{.70710678\pgfplotmarksize}{.70710678\pgfplotmarksize}}%
  \pgfpathmoveto{\pgfqpoint{-.70710678\pgfplotmarksize}{.70710678\pgfplotmarksize}}%
  \pgfpathlineto{\pgfqpoint{.70710678\pgfplotmarksize}{-.70710678\pgfplotmarksize}}%
  \pgfusepathqstroke
} 

\definecolor{blue1}{RGB}{31,119,180}
\definecolor{orange1}{RGB}{255, 127, 14}
\definecolor{green1}{RGB}{44, 160, 44}
\definecolor{red1}{RGB}{255, 102, 102}
   
\title{\LARGE \bf
A Teacher-Student Markov Decision Process-based Framework \\ for Online Correctional Learning*
}

\author{In\^{e}s Louren\c{c}o, Rebecka Winqvist, Cristian R. Rojas, and Bo Wahlberg$^{1}$
\thanks{*This work was supported by the Wallenberg AI, Autonomous Systems and Software Program (WASP), the Swedish Research Council Research Environment NewLEADS under contract 2016-06079, and the Digital Futures project EXTREMUM.}
\thanks{$^{1}$The authors are with the Division of Decision and Control Systems, KTH Royal Institute of Technology, Stockholm, Sweden
        {\tt\small \{ineslo,rebwin,crro,bo\}@kth.se}}%
}

\begin{document}

\maketitle
\thispagestyle{empty}
\pagestyle{empty}

\begin{abstract}
A classical learning setting typically concerns an agent/student who collects data, or observations, from a system in order to estimate a certain property of interest. \textit{Correctional learning} is a type of cooperative teacher-student framework where a teacher, who has partial knowledge about the system, has the ability to observe and alter (correct) the observations received by the student in order to improve the accuracy of its estimate. In this paper, we show how the variance of the estimate of the student can be reduced with the help of the teacher. We formulate the corresponding \textit{online} problem -- where the teacher has to decide, at each time instant, whether or not to change the observations due to a limited budget -- as a Markov decision process, from which the optimal policy is derived using dynamic programming. We validate the framework in numerical experiments, and compare the optimal online policy with the one from the \textit{batch} setting.

\end{abstract}


\section{Introduction}
\label{sec:intro}

With the rapid growth of smart systems and IoT, we are able to collect enormous amounts of data like never before. These data may range from medical images captured by camera sensors to distance measures from e.g. lidars and radars. From this data, agents can learn to perform tasks such as cancer detection and prognosis \cite{cancer_detection,cancer_prognosis}, and autonomous driving~\cite{autonomous_driving}. These are only some examples and the application domain is much more extensive.

In the Oxford dictionary \cite{oed:learn}, the term \textit{learning} is defined as the ``acquisition of knowledge or skills through study, experience, or being taught". In this work we consider a combination of the latter two; ``experience'', by using dynamic programming to train a teacher, and ``being taught'', by letting the trained teacher transfer its knowledge to a student agent. Setups that involve the presence of aiding expert agents are commonly denoted cooperative learning problems.

Cooperative problems play an important role in our lives. Indeed, most tasks we perform require some sort of collaboration; acquiring a new skill, such as learning how to drive, social learning, search and rescue operations, and much more. 
Solving these tasks is, however, not trivial, and the aid of an external agent can be very helpful. 
In the literature, the two most famous paradigms of cooperative learning that use a teacher-student framework are \textit{learning from demonstration} \cite{lfd} and \textit{imitation learning} \cite{imitation_learning}, in which the role of the teacher is to accelerate the learning of the student by means of showing it the optimal policy. 


In this work, we consider a different teacher-student paradigm. We study how the teacher can assist the student by intercepting, and altering, the data collected by the student. This approach is denoted \textit{correctional learning} and was recently proposed by \cite{lourencco2021cooperative}, in an effort to tackle the problem of assisting an agent in situations where transmitting knowledge directly might be impossible or undesirable. 
This correctional learning framework opens up for new possibilities. In the traditional learning setting, helping an agent learn a policy, the parameters of a system, or the state of the world, are some of the potential applications. These problems are called reinforcement learning, system identification, and filtering, respectively. Other examples are manual output-error correction of machine learning models, cooperative learning for task-performing, and estimation of user preferences and ratings. 
Alternatively, the correctional framework may be viewed as a means for diversifying the information presented to a user -- in social media applications and search engines, it could tackle the growing issue of echo chambers or confirmation bias, and the spreading of ``fake news''. Financial applications are another field of interest, in which the framework could be used to e.g. influence an investor's market state predictions for stock portfolio allocation.

In most of these fields, however, immediate (online) action is typically required as observations arrive sequentially, due to the need of a learning process that adapts and rapidly changes. Online algorithms often make learning faster and computationally cheaper. In this paper, we thus present the \textit{online} correctional learning framework where the teacher has to decide, at each time instant, whether or not to alter the corresponding observation.

The research question we answer in this paper is then:

\begin{center}
\textit{How should a teacher modify, at each instant and under budget constraints, the data received by a student in order to assist its learning process?}
\end{center}


The main contributions of this paper are as follows.
\begin{itemize}[itemsep=0pt]
\item Computation of a theoretical bound for how much the teacher can improve the estimation of the student in the case of discrete systems.
\item Formulation of a \textit{Markov decision process} (MDP) for the correctional learning framework performed in an online setting.
\item Demonstration of the results in two numerical experiments; in particular, the optimal policy of the teacher obtained using dynamic programming.
\item Comparison of the proposed online correctional learning framework with the batch framework.
\end{itemize}

The rest of the paper is organised as follows. In Section~\ref{sec:prob_form}, the correctional learning problem is formulated. In Section~\ref{sec:bounds}, we derive bounds for how much the teacher can help the student, and, in Section \ref{sec:online}, the proposed algorithm for performing online correctional learning is derived. Finally, Section \ref{sec:results} validates the presented methods in numerical experiments and Section \ref{sec:conclusions} concludes the paper.

\subsection{Related work}
\label{ssec:relatedwork}
Incorporating external information in decision-making is a commonly studied problem, not least in the teacher-student frameworks previously mentioned \cite{lfd,imitation_learning}. Despite being of a similar nature, our proposed framework differs from these in its nature of correcting the observations. 
Below are examples of areas that inspire our work and which resemble the correctional behavior of the teacher when selecting the inputs to intercept and alter.

\textit{Feature selection} is a technique used in learning and classification tasks to select the most relevant and non-redundant features to improve learning \cite{kira1992practical}. Similarly to our correctional learning framework, this family of methods has seen a shift from batch to online techniques \cite{wang2013online} -- which represent a more promising group of efficient and scalable machine learning algorithms for large-scale applications.

Our work is positioned around other important methods such as \textit{input design for system identification} \cite{pronzato2008optimal, hjalmarsson2009system}, where the input signals are designed to guarantee a certain model accuracy; \textit{active learning} \cite{aggarwal2014active}, where the learner queries the teacher for the desired labels; \textit{counterfactual explanations} \cite{verma2020counterfactual}, which is a branch within explainable artificial intelligence that uses feature importance to explain how a small perturbation of an input datapoint affects the output of a machine learning model; and \textit{learning with side information} \cite{kuusela2004learning}, in which additional information is provided to the learner to help its learning process. The concept of side information is also very popular in information theory (in connection to communication problems) \cite{cover2006elements}. 
%
%
Examples of other areas are \textit{active fault diagnosis} \cite{heirung2019input}, consisting of the design of an input signal for minimizing the time and energy required to detect and isolate faults in the outputs of a system; \textit{anomaly detection} \cite{chandola2009anomaly}, which aims to improve the performance of the model by removing anomalies from the training sample; and \textit{controlled sensing} \cite{krishnamurthy_2016}, where the decision-maker can choose at each time instant which sensor to use to obtain the next measurement.

In the next sections we show how we use these techniques as motivation and inspiration to create a simple and efficient online mechanism for sequentially correcting observations in a wide variety of applications. 

\begin{figure}
\centering

\tikzset{every picture/.style={line width=0.75pt}} 

\begin{tikzpicture}[x=0.75pt,y=0.75pt,yscale=-1,xscale=1]

\draw    (256,65.5) -- (311.17,65.66) ;
\draw [shift={(314.17,65.67)}, rotate = 180.16] [fill={rgb, 255:red, 0; green, 0; blue, 0 }  ][line width=0.08]  [draw opacity=0] (8.93,-4.29) -- (0,0) -- (8.93,4.29) -- cycle    ;
\draw    (457.17,64.67) -- (516.67,64.98) ;
\draw [shift={(519.67,65)}, rotate = 180.31] [fill={rgb, 255:red, 0; green, 0; blue, 0 }  ][line width=0.08]  [draw opacity=0] (8.93,-4.29) -- (0,0) -- (8.93,4.29) -- cycle    ;
\draw    (320.67,130.33) -- (320.96,74.4) ;
\draw [shift={(320.97,71.4)}, rotate = 450.3] [fill={rgb, 255:red, 0; green, 0; blue, 0 }  ][line width=0.08]  [draw opacity=0] (8.93,-4.29) -- (0,0) -- (8.93,4.29) -- cycle    ;
\draw    (384,130.5) -- (320.67,130.33) ;
\draw   (314.97,65.4) .. controls (314.97,62.08) and (317.66,59.4) .. (320.97,59.4) .. controls (324.29,59.4) and (326.97,62.08) .. (326.97,65.4) .. controls (326.97,68.71) and (324.29,71.4) .. (320.97,71.4) .. controls (317.66,71.4) and (314.97,68.71) .. (314.97,65.4) -- cycle ; \draw   (314.97,65.4) -- (326.97,65.4) ; \draw   (320.97,59.4) -- (320.97,71.4) ;
\draw    (326.67,65) -- (380,64.53) ;
\draw [shift={(383,64.5)}, rotate = 539.49] [fill={rgb, 255:red, 0; green, 0; blue, 0 }  ][line width=0.08]  [draw opacity=0] (8.93,-4.29) -- (0,0) -- (8.93,4.29) -- cycle    ;
\draw [color={rgb, 255:red, 208; green, 2; blue, 27 }  ,draw opacity=1 ][line width=1.5]    (469.19,44.89) -- (480.67,60) ;
\draw   (198,45) -- (256.35,45) -- (256.35,85) -- (198,85) -- cycle ;
\draw   (384,109) -- (442.35,109) -- (442.35,149) -- (384,149) -- cycle ;
\draw   (383,38) -- (456,38) -- (456,88.5) -- (383,88.5) -- cycle ;

\draw (227,56) node  [font=\small,color={rgb, 255:red, 0; green, 0; blue, 0 }  ,opacity=1 ] [align=left] {System};
\draw (413.3,120.47) node  [font=\small,color={rgb, 255:red, 0; green, 0; blue, 0 }  ,opacity=1 ] [align=left] {Teacher};
\draw (228,74) node  [font=\footnotesize]  {$\theta _{0}$};
\draw (355,67) node  [font=\normalsize,color={rgb, 255:red, 74; green, 144; blue, 226 }  ,opacity=1 ]  {\color{blue}{$ \begin{array}{l} 
\tilde{\mathcal{D}} =\\
\left\{\tilde{y}_{k}\right\}_{k=1}^{N} 
\end{array}$}};
\draw (287,67) node  [font=\normalsize]  {$ \begin{array}{l}
\mathcal{D} =\\
\{y_{k}\}_{k=1}^{N}
\end{array}$};
\draw (276.46,24.67) node  [font=\footnotesize,color={rgb, 255:red, 0; green, 0; blue, 0 }  ,opacity=1 ] [align=left] {\begin{minipage}[lt]{48.99pt}\setlength\topsep{0pt}
\begin{center}
Original\\observations
\end{center}

\end{minipage}};
\draw (416,138) node  [font=\footnotesize]  {$\theta _{0} ,y_{k},b_k$};
\draw (354.46,122.67) node  [font=\footnotesize,color={rgb, 255:red, 74; green, 144; blue, 226 }  ,opacity=1 ] [align=left] {{\color{blue}Correction}};
\draw (350.46,23.67) node  [font=\footnotesize,color={rgb, 255:red, 0; green, 0; blue, 0 }  ,opacity=1 ] [align=left] {\begin{minipage}[lt]{48.99pt}\setlength\topsep{0pt}
\begin{center}
Corrected\\observations
\end{center}

\end{minipage}};
\draw (420,50) node  [font=\small,color={rgb, 255:red, 0; green, 0; blue, 0 }  ,opacity=1 ] [align=left] {Student};
\draw (421.23,73.15) node  [font=\footnotesize,color={rgb, 255:red, 0; green, 0; blue, 0 }  ,opacity=1 ] [align=left] {\begin{minipage}[lt]{45.8pt}\setlength\topsep{0pt}
\begin{center}
(Estimation \\algorithm)
\end{center}

\end{minipage}};
\draw (487,50) node  [font=\normalsize]  {$\hat{\theta }\rightarrow \textcolor[rgb]{0.29,0.56,0.89}
{{\color{blue}\tilde{\theta }} }$};
\draw (485.46,25) node  [font=\footnotesize,color={rgb, 255:red, 0; green, 0; blue, 0 }  ,opacity=1 ] [align=left] {\begin{minipage}[lt]{40.82pt}\setlength\topsep{0pt}
\begin{center}
Estimate
\end{center}

\end{minipage}};

\end{tikzpicture}

\caption{Schematic representation of the correctional learning framework.}
\label{fig:onlineCL}
\end{figure}
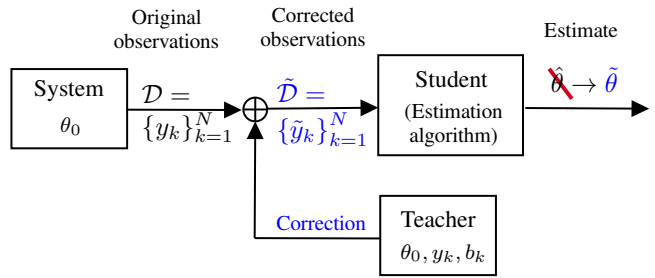
\section{Problem Formulation}
\label{sec:prob_form}

In this section, we define the notation and formally introduce the correctional learning problem: a teacher, who has knowledge about a system of interest, aims to help a student's learning process by altering the data it receives. 

\subsection{Notation}

All vectors are column vectors and inequalities between
vectors are considered element-wise. The $i$-th element of a vector $v$ is $[v]_i$. A generic probability density (or mass) function is denoted as $p(\cdot)$ and $\mathbb{P}\{\cdot\}$ is the probability of event $\cdot$. 
Unif($\cdot$) is the uniform distribution. The vector of ones is denoted as $\mathds{1}$, the set of natural numbers including zero as $\mathbb{N}_0$, and the indicator function as $I(\cdot)$. 
The $\ell_1$ norm of a vector $v$ is denoted as $\|v\|_1$ and card$(S)$ is the cardinality of set~$\mathcal{S}$.

\subsection{Introduction to Correctional Learning}

A learning agent (student) is sequentially collecting information from a source in the form of observations $\mathcal{D} = \{y_k\}_{k=1}^N$, $y_k \in \mathcal{Y}$, throughout $N$ time steps, and estimating characteristics of interest $\hat{\theta}$ about that system.
An expert agent (teacher) has more knowledge about the system and its goal is to assist the estimation process of the student. However, for several reasons it might be impossible or undesirable for the teacher to transmit its knowledge directly to the student. For example, the expert's knowledge might be too abstract (teaching someone to drive), or too complex to be transmitted, the communication might be restricted due to privacy concerns, or the teacher and student might operate in different model classes or parameterizations \cite{lourencco2021cooperative}.
The teacher, therefore, instead has the ability to interfere by intercepting, and altering, the observations collected by the student to $\tilde{\mathcal{D}} = \{\tilde{y}_k\}_{k=1}^N$.
A schematic representation is shown in Figure \ref{fig:onlineCL}.
Improving the estimation thus means obtaining an altered estimate $\tilde{\theta}$ closer to the true estimate ($||\tilde{\theta}_N-\theta_0|| \leq ||\hat{\theta}_N-\theta_0||$), or which converges to the true one in less iterations (var$\{\tilde{\theta}\} < \text{var}\{\hat{\theta}\}$).
This is of particular importance when the teacher has studied the system of interest for a longer time than the student and thus has a more accurate estimate. However, if this estimate is not perfect, rather than giving the student its estimate, it instead corrects the information acquired by the student to better match its own. One can also imagine that the student uses few training examples to study the behaviour of an agent, and might thus include exploratory actions in its analysis -- which biases its estimation and illustrate the importance of a teacher interfering to alter these.
%

Additionally, altering the data might be expensive or dangerous for privacy concerns. 
Therefore, correctional learning includes the budget constraint
\begin{equation}
 B( \mathcal{D} , \tilde{\mathcal{D}} ) \leq b,
 \label{eq:budget_constraint}
\end{equation}

\noindent where $B$ is a distance measure between two sequences which represents the budget, $b$, that the teacher has on how much it can interfere with the observations obtained from the system.
If the observations are discrete, $B$ can be defined as the $l_1$-norm $ \frac{1}{N} \sum_{k=1}^N |y_k-\tilde{y}_k| \leq b$.

\subsection{Batch Correctional Learning}

In \textit{batch} correctional learning, multiple observations are intercepted simultaneously. This can be the case in, for example, a communication channel, where multiple bits can be delayed on the way from the source to the receiver. 
In \cite{lourencco2021cooperative}, the batch problem was solved by minimizing the distance between the true parameter and the empirical estimate computed by the student, $V(\theta_0,\tilde{\theta})$, according to the following optimization problem:
\begin{equation}
\begin{aligned}
\min_{\tilde{\mathcal{D}}  }  \quad & V(\theta_0, \tilde{\theta}) \\
\text{s.t.} \quad & \tilde{y}_k \in \mathcal{Y}, \text{ for all } \tilde{y}_k \in \tilde{\mathcal{D}},\\
 &  B( \mathcal{D} , \tilde{\mathcal{D}} ) \leq b,
\end{aligned}
\label{eq:batchopt}
\end{equation}
where $B$ is the distance measure from \eqref{eq:budget_constraint}.
%
In \cite{lourencco2021cooperative} it was shown that the resulting set $\tilde{D}$ of corrected observations was the optimal one for the case of binomial data, and by how much the variance of the corrected estimate is decreased compared to the original one. In this paper, we formulate an MDP to solve the problem in an online setting and for an extensive variety of applications.
%

\section{Correctional learning bounds \\for discrete systems}
\label{sec:bounds}
In this section we analyse how much the teacher can effectively help the student, for the case when the observations are discrete.  


The following theorem relates the estimates of the mean values of two sequences of observations -- the original, $Y/N$, and the corrected one by the teacher, $\tilde{Y}/N$. Knowing how much the variance of the corrected estimate is decreased is a measure of how much the teacher can help reducing the error of the estimation of the student.\\


\begin{theorem}[Variance decrease of the altered estimate]
Let $X_1, \hdots, X_N$ be i.i.d. random variables in $\{0,1,\hdots,M\}$ with mean $\mu$, and $Y = X_1 + \cdots + X_N$. Let $B\in\{0,1,\hdots,N\}$, and 
    \begin{equation}
        \tilde{Y} = \underset{\{Z\in\{0,\hdots,N\} : |Y-Z| \leq B\}}{\arg\min} |Z-N\mu|.
    \end{equation}
    Then, 
    \begin{equation}
    \textnormal{var}[\tilde{Y}/N] \leq M^2\exp\left[-\frac{2B^2}{NM^2}\right]. 
    \end{equation}
    Let us further assume that $X_i \sim \text{Unif} \left(\{0, \hdots, M\}\right)$ (this assumption is not crucial but provides a special case that is easier to understand). Then,
    \begin{equation}
    \begin{aligned}
        \frac{\textnormal{var}[\tilde{Y}/N]}{\textnormal{var}[Y/N]} \leq \frac{6M}{5M+1}\exp\left[-\frac{2B^2}{NM^2}\right].
    \end{aligned}
\end{equation}
\label{thm:variance_bound}
\end{theorem}

The proof of Theorem \ref{thm:variance_bound} is given in Appendix \ref{app:bound}. This theorem provides an upper bound for the decrease in variance of the estimate computed by the student due to the help of the teacher, according to its budget $B$. It implies that:
\begin{enumerate}[leftmargin=*,itemsep=4pt, label=\roman*)]
    \item the teacher's ability to improve the learning of the student increases with its budget; 
    \item for a given budget, the improvement becomes less important as $N$ grows. This is reasonable, since the average deviation of $Y/N$ around $\mu$ is of order $\mathcal{O}(1/\sqrt{N})$, while the improvement due to the teacher can be at most $B/N$;
    \item for a fixed budget and $N$, the improvement degrades as $M$ increases, since the variance of $Y/N$ increases with $M$, which makes it increasingly harder for a teacher to compensate for ``bad'' samples. 
\end{enumerate}
After computing by how much the teacher can improve the estimation process of the student in a discrete setting, we next propose a framework for how the teacher can achieve this by altering the observations in real time.

\section{Online Correctional learning}
\label{sec:online}

In this section we present a framework for computing an optimal online policy for the teacher. Unlike the batch case, the online setting is a more realistic scenario in which the observations are obtained sequentially and the expert has to make, at each time instant, the decision of whether or not to change the current observation. 

\subsection{Formulation of the Markov Decision Process}

We are now ready to present the second main result of our paper, which is the formulation of an MDP to describe the teacher's policy for the online correctional learning problem when the student samples discrete observations $y_k \in \{0,1,\dots,M\} = \mathcal{Y}$ from a system to estimate the true parameter $\theta_0$:

\begin{align*}
\noindent
\textbf{States:} \quad & s =(x_{1:k}, b_k, y_k) \\
\textbf{Actions:} \quad & a=\{\text{keep } y_k, \text{change to } \bar{y}_k\}\\
\textbf{Time-horizon:} \quad &N \\
\textbf{Reward function:} \quad & -||\tilde{\theta}_N \; - \; \theta_0 ||_1 \\
\textbf{Constraint:} \quad &\text{number of times the action }\\
 & \text{``change to } \bar{y}_k" \text{is taken} \leq b\\
\textbf{Transition probabilities:} \quad & \text{see \eqref{eq:transition_probs}}.
\end{align*}

In more details:
\begin{itemize}[leftmargin=*,itemsep=8pt]
    \item  \textit{States}: The states $\mathcal{S}$ of the MDP are tuples containing: \textit{i)} $x_{1:k}$ -- an $M\times 1$ vector with the number of times each observation has been seen until time $k$; \textit{ii)} $b_k \in \mathbb{N}_0$ -- the current budget left to use at time $k$; \textit{iii)} $y_k$ -- the observation received at time $k$.
The number of states, card$(\mathcal{S})$, is finite and upper bounded by  $N^{M+1} b$. However, the constraint $\sum_{l=1}^M [x]_l\leq N$ renders many of these states invalid, which results in a much smaller and tight upper bound: card$(\mathcal{S}) \leq \text{card}(x) b N$. Here, card$(x)$ can be computed using \textit{multiset coefficients} as 
\begin{equation}
\textnormal{card}(x) = \sum_{n=1}^N  \llrrparen{\psymbol{M}{n}} = \sum_{n=1}^N \frac{M(M+1)\hdots(M+n-1)}{n!},
\end{equation}
which are the $N$-permutations of $M$ with repetitions and which satisfy the previous constraint. 

\item \textit{Terminal states}: These are the ones where all the observations have been received, that is, where
\begin{equation}
\sum_{l=1}^M \; [x]_l= N.
\end{equation}

\item \textit{Actions}: The possible actions are to keep the last observation $y_k$ or change it to a certain value $\tilde{y}_k \in \mathcal{Y}$. The number of actions is card$(A) = M$.

\item \textit{Reward function}: The reward is zero in all states except in the terminal states, where it is inversely proportional to the error of the estimate computed after the teacher's alterations.

\item \textit{Transition probabilities}: If the action is ``keep $y_k$", the next state depends, with probability $p(y_{k+1})$, on the next received observation $y_{k+1}$. The value of the next state is obtained by simply replacing the last value of the previous state $y_k$ by the new observation received, and adding one to that entry of the vector $x$, $[x']_{y_{k+1}} = [x]_{y_{k+1}}+1$.
If the action is ``change to $\tilde{y}_k$", the next state will have the same probability as in the previous case, where one is added to $[x]_{y_{k+1}}$. However, it will now have a one subtracted from the previous observation in $[x]_{y_k}$ and a one added in $[x]_{\tilde{y}_k}$ (since $y_k$ was altered to $\tilde{y}_k$), as well as a budget of $b_{k+1}=b_k-1$. We can write these mathematically as: 

\begin{equation}
\begin{aligned}
& \mathbb{P}\{(x',b,y_{k+1}) \; |  \; s = (x,b,y_k), a = ``\text{keep } y_k" \} = p(y_{k+1}), \\
& \quad \quad  \text{where } 
[x']_{y_{k+1}} = [x]_{y_{k+1}} +1 \\
& \mathbb{P}\{(x',b-1,y_{k+1}) \; | \; s = (x,b,y_k), a = ``\tilde{y}_k" \} = p(y_{k+1}), \\
& \quad \quad   
\text{where } [x']_{y_{k+1}} = [x]_{y_{k+1}} +1, [x']_{y_k} = [x]_{y_k}-1, \text{and }\\
& \quad \quad [x']_{\tilde{y}_k} = [x]_{\tilde{y}_k} +1, \\
& \text{and } \mathbb{P}\left\{\text{others}\right\} = 0.
\end{aligned}
\label{eq:transition_probs}
\end{equation}

\noindent Note that the chosen formulation of the states and actions satisfies the Markovian property.

\item \textit{Constraint}: The constraint is enforced by attributing an infinitely negative reward to transitions to states where the budget would be $b_{k+1}<0$.\\
\end{itemize}

The optimal policy for the online correctional learning problem represented as the previous MDP can be obtained using dynamic programming \cite{puterman2014markov}.\\

This framework can be translated to different scenarios by adjusting the reward function to a representative description of the student's goal in the task at hand.

\begin{remark}
This framework can also be used when the observations are continuous, by discretizing the observation space and changing the constraint to the total amount of correction $\sum_{k=1}^N |y_k-\tilde{y}_k| \leq b$.\\
\end{remark}






\pgfplotstableread{
1.  2.  3.  4.  5.  6.  7.  8.  9. 10. 11. 12. 13. 14. 15. 16. 17. 18. 19. 20. 21. 22. 23. 24. 25. 26. 27. 28. 29. 30. 31. 32. 33. 34. 35. 36. 37. 38. 39. 40. 41. 42. 43. 44. 45. 46. 47. 48. 49. 50.
0.4 1.  0.2 0.4 0.2 0.8 0.8 0.4 0.8 0.2 0.2 0.2 0.6 0.2 0.4 0.4 0.2 0.2 0.4 0.4 1.2 0.4 0.4 0.6 0.6 0.6 0.8 0.2 0.6 0.2 0.4 0.4 0.6 0.6 0.6 0.8 0.6 1. 0.4 0.6 0.6 0.8 0.6 0.8 1. 0.4 0.2 0.6 0.4 0.6
0.2 0.6 0.2 0.2 0.2 0.4 0.4 0.2 0.6 0.2 0.2 0.4 0.4 0.2 0.2 0.2 0.2 0.2 0.2 0.2 0.8 0.2 0.2 0.2 0.2 0.2 0.4 0.2 0.2 0.2 0.2 0.2 0.2 0.2 0.2 0.4 0.2 0.6 0.2 0.2 0.2 0.4 0.2 0.4 0.6 0.2 0.2 0.2 0.2 0.2
0.2 0.6 0.2 0.2 0.2 0.4 0.4 0.2 0.4 0.2 0.2 0.2 0.2 0.2 0.2 0.2 0.2 0.2 0.2 0.2 0.8 0.2 0.2 0.2 0.2 0.2 0.4 0.2 0.2 0.2 0.2 0.2 0.2 0.2 0.2 0.4 0.2 0.6 0.2 0.2 0.2 0.4 0.2 0.4 0.6 0.2 0.2 0.2 0.2 0.2
}\datasetinit
\pgfplotstabletranspose\datasetmultinomial{\datasetinit}

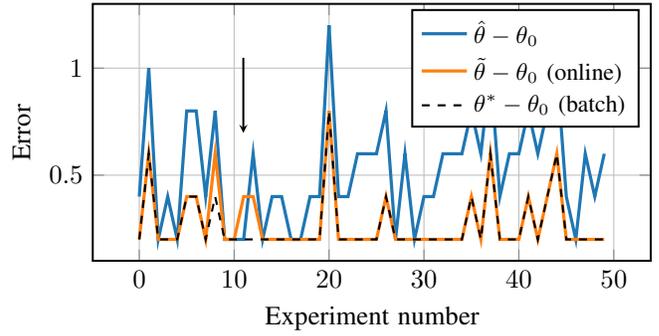
\begin{figure}[t!]
\centering
\begin{tikzpicture}
  \begin{axis}[
height=5cm,width=9cm,ylabel={Error}, xlabel={Experiment number}, grid, legend style={font=\small}, legend cell align={left},
		]
    \addplot[color=blue1,very thick]  table[x index=0, y index=2] \datasetmultinomial; 
    \addplot[color=orange1,very thick] table[x index=0,y index=3] \datasetmultinomial; 
    \addplot[color=black,thick, dashed] table[x index=0,y index=4] \datasetmultinomial; 
    \legend{$\hat{\theta}-\theta_0$,$\tilde{\theta}-\theta_0$ (online),$\theta^*-\theta_0$ (batch)}
  \end{axis}
  	\draw [-stealth](2,2.7) -- (2,1.7);
\end{tikzpicture}
    \caption{The blue line shows the estimation errors obtained during 50 experiments with the original sequence of observations ($b=0$). The orange line represents how much this error is decreased with the help of the teacher with $b=1$, that is, by changing one observation. The online policy learned by the teacher thus allows it to reduce the errors considerably, closely approaching the batch case shown by the dashed black line. The larger the budget $b$, the closer the orange and black curves become to $e_{min}$ \eqref{eq:emin}.}
    \label{fig:online_error}
\end{figure}

\section{Numerical Results}
\label{sec:results}

In this section, we validate the framework proposed in Section \ref{sec:online} by showing significant gains in using the teacher to improve the learning of the student in two different tasks.
We first consider the simple example of computing the mean of bi- and multi-nomial data, since it grants us, due to its simplicity, the derivation of explicit solutions for the batch and online settings, as well as a thorough analysis of the intrinsic workings of the framework. 
We then apply the proposed framework to a problem of biological parameter estimation, to illustrate its application to more complex scenarios.
The simulations were performed using Python 3.7 and a 1.90 GHz CPU. 

\subsection{Example: multinomial data}

Figure \ref{fig:online_error} presents the results of the MDP proposed in Section \ref{sec:online} for performing correctional learning in an online setting. 
The figure shows the error of the estimate of the student with, in blue, the original sequence of observations -- that is, without the help of the teacher -- and, in orange, the corrected sequence. The estimates $\theta$ are computed as the mean of the observations, $[\theta]_i=\sum_{k=1}^M I(y_k=i) /N$, and we consider that an observation is randomly sampled $N=5$ times from a multinomial distribution with parameter $\theta_0=[0.4;0.3;0.3]$, over 50 experiments. 
Unlike in the binomial case, where we can compute a closed form solution for the minimum attainable error (see Appendix \ref{app:binomial}), in the multinomial case this error is given by the batch error, which we computed using (8) from \cite{lourencco2021cooperative} and with the $l_1$-norm in the objective function for consistency. The minimum error, independent of $\hat{\theta}$ and $b$, can, however, be computed as 
\begin{equation}
e_{min}(N,\theta_0) =  \left\Vert \theta_0- \frac{[\theta_0 N]}{N} \right\Vert_1 = 0.2, 
\label{eq:emin}
\end{equation}

\noindent which is achieved by $\theta^*=[0.4; 0.4; 0.2]$ or $[0.4;0.2;0.4]$. In \eqref{eq:emin}, the brackets $[\cdot]$ without subscript mean rounding to the closest integer value, subject to the constraint $\mathds{1}^T \theta_{min} =  1$ where $\theta_{min} = \frac{[\theta_0 N]}{N}$.

Intuitively, one would expect the teacher's optimal policy to be delaying as much as possible spending its budget. In the binomial case the online policy learned,
\begin{fleqn}
\begin{equation}
\mu^*=
\begin{cases}
    a_k = \text{keep } y_k, \qquad \text{if }b_k\leq0 \text{ or }[x]_{y_k}\leq \big[[\theta_0]_{y_k} N \big], \\
    a_k =\text{alter to } \tilde{y}_k=1-y_k,\qquad \text{otherwise}, 
    \end{cases}
\label{eq:optpolicy}
\end{equation}
\end{fleqn}
is optimal since it coincides with the policy computed using batch correctional learning, as is shown in Appendix \ref{app:binomial}. In the multinomial case from Figure \ref{fig:online_error}, both differ only in a limited amount of scenarios, when a less expected sample that has a small reward is obtained. Note that in experiment 11 (marked with an arrow in the figure), the altered estimate $\tilde{\theta}$ is even worse than the original one, $\hat{\theta}$. The teacher chose to alter the fourth observation of $\underline{1}, \underline{2}, \underline{0}, \underline{\textbf{\textit{2}}}, \underline{?}$ to a 0 since the expected value was larger (receiving a 1 or a 2 at $k=5$ had a large probability and maximum reward), but the less likely observation, 0, was received instead. 

\begin{figure}[t!]
\centering
\begin{tikzpicture}
  \begin{axis}[
height=4.5cm,width=7cm,ylabel={$\text{var}\{\tilde{\theta\}}$}, ytick={0,0.02,0.04}, scaled y ticks = false, y tick label style={/pgf/number format/fixed}, xlabel={Number of observations, $N$}, grid, enlarge x limits=0.05, ymin =0,
		]
    \addplot[color=blue1,very thick] coordinates {
	(5,  0.045)
	(10,  0.025)
	(15,  0.0205)
	(20, 0.013)
	(25, 0.008)
};
    \addplot[color=orange1,very thick] coordinates {
	(5,  0.016)
	(10,  0.009)
	(15,  0.008)
	(20, 0.006)
	(25, 0.003)
};
    \addplot[color=green1,very thick] coordinates {
	(5,  0.01)
	(10,  0.005)
	(15,  0.003)
	(20, 0.002)
	(25, 0.002)
};
    \legend{b=0,b=1,b=2}
  \end{axis}
\end{tikzpicture}
\caption{Reduction of the variance of the estimate for increasing budgets $b$ as $N$ increases. The case $b = 0$ corresponds to when the teacher cannot assist the student.}
\label{fig:online_variances_budget}
\end{figure}
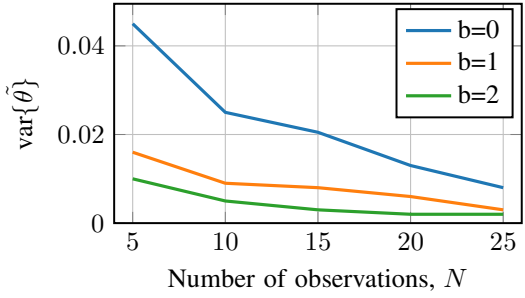



Figure \ref{fig:online_variances_budget} shows that, as expected, the variance of the estimate decreases as the number of observations increases. However, as the budget of the teacher increases, the variance is further decreased. This result illustrates the conclusions from Theorem \ref{thm:variance_bound}.

\subsection{More complex example: biological parameter~estimation} 

Biological internal models have taken a major role in the exploration and validation of neuroscientific theories, 
e.g., when it comes to understanding the role of the Cerebelum in motor control \cite{welniarz2021forward}, or predicting and treating neurological diseases \cite{lanillos2020review}. 
In the next example, we apply online correctional learning to a scenario where the student estimates biological neural parameters $\theta$ from observing the actions $a$ performed by biological agents, such as animals or other humans. The observations $\mathcal{Y}$ are in this case the history $\mathcal{H}$ of actions $a$, and problems like this are called inverse problems \cite{engl2009inverse}.
Using a forward model of behaviour, the actions distribution $p(a|\theta)$ can be computed, and, from there, the likelihood $L(\mathcal{H};\theta)$. Maximizing this likelihood (or minimizing its negative value), originates the student's estimate $\hat{\theta}$. 
In \cite{lourencco2021biologically}, the authors use this inverse method to estimate parameters of neural time perception mechanisms. The model \cite{lourencco2020teaching} proposed to replicate these mechanisms generates the data from Figure \ref{fig:time_perception_distributions} over 2000 episodes. As the student observes the actions of the animal throughout the task, the teacher uses the framework proposed in Section \ref{sec:online} to correct certain observations (obtaining a corrected history $\tilde{\mathcal{H}}$), in order to improve the student's estimation of the animal's biological parameters -- correct the sampled distribution to be more similar to the distribution from Figure \ref{fig:time_perception_distributions} that corresponds to the true parameter $\theta_0$.
In this example, the observations are the actions performed by the animal ($M$ is the number of possible actions), and $\theta_0$ is the number of microstimuli of its time perception mechanism \cite{ludvig2008stimulus}. The reward function is given by the difference between $\theta_0$ and $\tilde{\theta}$, where the latter is computed from
\begin{equation}
   \tilde{\theta}= \arg \min_\theta - L(\tilde{\mathcal{H}};\theta),
   \label{eq:paramEst}
\end{equation}
and which corresponds to the difference between the student's corrected estimate and the true parameter.

\pgfplotstableread{
5 0.287012987012987  0.7077922077922078   0.0025974025974025974   0.0025974025974025974
}\dataseta

	\pgfplotstableread{
5  0.12847790507364976   0.7422258592471358   0.04991816693944354   0.07937806873977087
}\datasetb

	\pgfplotstableread{
5  0.1295754026354319   0.7445095168374817   0.0036603221083455345   0.12225475841874085
   }\datasetc

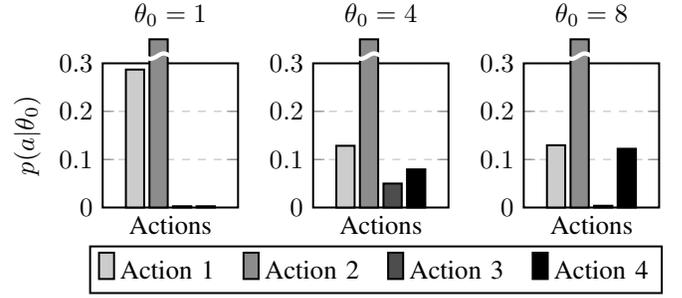
\begin{figure}[t!]
\centering
\begin{tikzpicture}
\pgfplotsset{
every axis legend/.append style={ at={(-0.1,-0.6)}, anchor=south west,legend columns = -1}}
    \begin{groupplot}[
        group style={
            group size=3 by 1,
            x descriptions at=edge bottom,
            vertical sep=15pt,
        },
        grid,
        height=5cm,
       	width=15cm,
        grid style=dashed,
        legend style={/tikz/every even column/.append style={column sep=0.3cm}}
    ]
    \nextgroupplot[
			ybar,
            height=3.5cm,
       	    width=6cm,
       	    x=0.9cm,
           ymin=0,
            enlarge x limits=0.15,
           ymax=0.3,        
           ylabel={$p(a|\theta_0)$},
          xtick={0},
          xlabel={Actions},
          xlabel shift=-5pt,
          major x tick style = {opacity=0},
          title={$\theta_0=1$},
          title style={yshift=1ex},
          minor tick length=0ex,
          bar width=0.27,
                    restrict y to domain*=0:0.35,
          clip=false,
          after end axis/.code={ 
            \draw [ultra thick, white, decoration={snake, amplitude=1pt}, decorate] (rel axis cs:0,1.05) -- (rel axis cs:1,1.05);
        },
          legend image code/.code={
       \draw [#1] (0cm,-0.1cm) rectangle (0.2cm,0.25cm); },
            ]
  \addplot[fill=black!20] table[x       index=0,y index=1] \dataseta; 
   \addplot[fill=black!45]		 table[x    index=0,y index=2] \dataseta; 
   \addplot[fill=black!70] table[x    index=0,y index=3] \dataseta; 
   \addplot[fill=black!100] table[x    index=0,y index=4] \dataseta;  
    \legend{Action 1,Action 2,Action 3,Action 4},

    \nextgroupplot[
			ybar,
        height=3.5cm,
       	width=7cm,
       	x=0.9cm,
           ymin=0,
         enlarge x limits=0.15,
           ymax=0.3,        
          title={$\theta_0=4$},
          title style={yshift=1ex},
          xtick={0},
          xlabel={Actions},
          xlabel shift=-5pt,
          major x tick style = {opacity=0},
          minor tick length=0ex,
          bar width=0.27,
                    restrict y to domain*=0:0.35,
          clip=false,
          after end axis/.code={ 
            \draw [ultra thick, white, decoration={snake, amplitude=1pt}, decorate] (rel axis cs:0,1.05) -- (rel axis cs:1,1.05);
        },
            ]
  \addplot[fill=black!20] table[x       index=0,y index=1] \datasetb; 
   \addplot[fill=black!45] table[x    index=0,y index=2] \datasetb; 
   \addplot[fill=black!70] table[x    index=0,y index=3] \datasetb; 
   \addplot[fill=black!100] table[x    index=0,y index=4] \datasetb;  
        
    \nextgroupplot[
			ybar,
        height=3.5cm,
       	width=7cm,
           ymin=0,
           x=0.9cm,
           ymax=0.3,        
           title={$\theta_0=8$},
           title style={yshift=1ex},
           xtick={0},
            xlabel={Actions},
          xlabel shift=-5pt,
          major x tick style = {opacity=0},
          minor tick length=0ex,
          bar width=0.27,
          enlarge x limits=0.15,
          restrict y to domain*=0:0.35,
          clip=false,
          after end axis/.code={ 
            \draw [ultra thick, white, decoration={snake, amplitude=1pt}, decorate] (rel axis cs:0,1.05) -- (rel axis cs:1,1.05);
        },
            ]
  \addplot[fill=black!20] table[x       index=0,y index=1] \datasetc; 
   \addplot[fill=black!45] table[x    	index=0,y index=2] \datasetc; 
   \addplot[fill=black!70] table[x    index=0,y index=3] \datasetc; 
   \addplot[fill=black!100] table[x    index=0,y index=4] \datasetc;  
    \end{groupplot}
\end{tikzpicture}
\caption{Effect of the parameter $\theta_0$ on the behaviour $a$ of the agent. These statistics are the base of our model and were computed over one training simulation with 2000 episodes. For the rest of this example we assume that $\theta_0=4$ is the true parameter, being the observations sampled from the distribution on the middle.}
\label{fig:time_perception_distributions}
\end{figure}


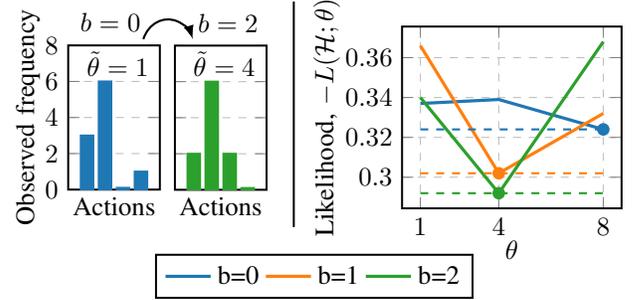
\begin{figure}[t!]
\centering
\begin{tikzpicture}
  \begin{groupplot}
  [group style={group size=3 by 1, horizontal sep=1.5cm},no markers]
  \draw[thick, black, -latex] (1,2) arc  [ start angle=160,
        end angle=20,
        x radius=0.35cm,
        y radius =0.4cm ] ;
    \draw[black] (3,-0.5) -- (3,2.5);
    
    \nextgroupplot[
    width=2.8cm, height=3.5cm,
	x tick label style={
		/pgf/number format/1000 sep=},
	enlarge x limits=0.15,
	ybar,
	grid, grid style=dashed,
	bar width=3pt,
    xlabel={Actions},
	ylabel={Observed frequency},
	ylabel shift=-5pt,
	xlabel shift=-5pt,
	ymin=0,
	ymax=8,
    xtick={-1},
    enlarge x limits=0.35,
	bar width = 0.7,
	title={$b=0 \;$},
	title style={yshift=-1ex},
]
\addplot[color=blue1, fill]
	coordinates {(1,3) (2,6)
		 (3,0.1) (4,1)};
   \node[align=center, text=black] at (2.7,7) {$\tilde{\theta}=1$};
   
    \nextgroupplot[xshift=-1.3cm,
    width=2.8cm, height=3.5cm,
	x tick label style={
		/pgf/number format/1000 sep=},
	enlarge x limits=0.15,
	ybar,
	grid, grid style=dashed,
	bar width=3pt,
    xlabel={Actions},
	ylabel shift=-5pt,
	xlabel shift=-5pt,
	ymin=0,
	yticklabels={,,},
	ymax=8,
	xtick={-1},
	enlarge x limits=0.35,
	bar width = 0.7,
	title={ $\; \; b=2$},
	title style={yshift=-1ex}
]
\addplot[color=green1, fill]
	coordinates {(1,2) (2,6) 
		(3,2) (4,0.1)};
 \node[align=center, text=black] at (2.7,7) {$\tilde{\theta}=4$};
		   
    \nextgroupplot[height=4cm,width=4.5cm,xshift=0.3cm,ylabel={Likelihood, $-L(\mathcal{H};\theta)$}, xlabel={$\theta$}, xtick={1,4,8}, xlabel shift=-5pt, ylabel shift=-5pt, legend style={at={(-0.4,-0.25), line width=0.5pt},
		anchor=north,legend columns=-1}, grid, grid style=dashed,
		]
    \addplot[color=blue1,very thick] coordinates {
	(1,  0.337)
	(4,  0.339)
	(8,  0.324)
};
    \addplot[color=orange1,very thick] coordinates {
	(1,  0.366)
	(4,  0.302)
	(8,  0.332)
};
    \addplot[color=green1,very thick] coordinates {
	(1,  0.340)
	(4,  0.292)
	(8,  0.368)
};
    \addplot [color=blue1,only marks,mark=x,  mark size=2pt] coordinates { (8,  0.324) };
    \addplot [color=orange1,only marks,mark=x,  mark size=2pt] coordinates { (4,  0.302) };
    \addplot [color=green1,only marks,mark=x,  mark size=2pt] coordinates	{(4,  0.292) };
    \addplot [domain=1:8,thick,dashed, blue1, ]{0.324};
    \addplot [domain=1:8,thick,dashed, orange1, ]{0.302};
    \addplot [domain=1:8,thick,dashed, green1, ]{0.292};
    \legend{b=0,b=1,b=2}
  \end{groupplot}
\end{tikzpicture}
 \caption{On the left is shown an example of the action distributions seen by the student before and after the teacher corrections, for a budget of $b=2$ for $N=10$. On the right, is shown the effect of these corrections on the corresponding negative likelihood, with the estimated parameter (the one that has the minimum negative likelihood) marked with a dot. For a budget of $b=2$, the action distribution is closer to the true distribution corresponding to $\theta_0=4$ from Figure \ref{fig:time_perception_distributions}, and therefore the minimum negative likelihood parameter becomes $\tilde{\theta}=\theta_0=4$ instead of $\tilde{\theta}=1$.}
\label{fig:CL_time_perception_corrections}
\end{figure}

\pgfplotstableread{
0 40.4  34.9  25.3   22.8   19.  17.  26.6   8.42    18.18   35.02
1  0.  0.3   0.7   0.6  2.2  2.5        1.05   0.95    0.1    2
2   0  0   0   0   0   0.2            0.033  0.0745  0       0.1075
}\dataseterrors

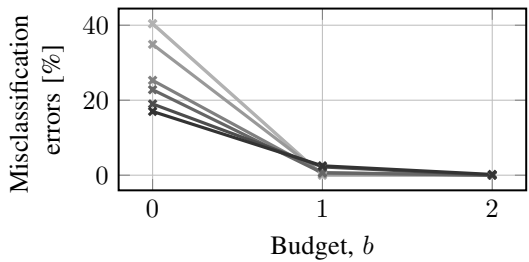
\begin{figure}[t!]
\centering
\begin{tikzpicture}
  \begin{axis}
  [height=4cm,width=7cm,xlabel={Budget, $b$}, ylabel style={align=center,text width=2.5cm}, ylabel={Misclassification \newline errors [\%]}, grid,xtick={0,1,2}]
    \addplot[mark=x, very thick, black!20] table[x       index=0,y index=1] \dataseterrors; 
    \addplot[mark=x, very thick, black!30] table[x       index=0,y index=1] \dataseterrors; 
    \addplot[mark=x, very thick, black!40] table[x       index=0,y index=2] \dataseterrors; 
    \addplot[mark=x, very thick, black!50] table[x       index=0,y index=3] \dataseterrors; 
    \addplot[mark=x, very thick, black!60] table[x       index=0,y index=4] \dataseterrors; 
    \addplot[mark=x, very thick, black!70] table[x       index=0,y index=5] \dataseterrors; 
    \addplot[mark=x, very thick, black!80] table[x       index=0,y index=6] \dataseterrors; 
  \end{axis}
\end{tikzpicture}
\caption{The percentage of times that the student incorrectly estimates the true parameter decreases with the increase of the budget of the teacher, averaged over $1000$ experiments and plotted for different sample sizes $N$.}
\label{fig:CL_time_perception_accuracy}
\end{figure}

The left plot of Figure \ref{fig:CL_time_perception_corrections} shows the total number of times each action was observed by the student in a certain experiment, and the corresponding corrections of the teacher as its budget increases. The right plot illustrates how these corrections alter the likelihood of the student estimating each parameter, converging to $\tilde{\theta}=\theta_0=4$ for budgets $b$ larger than changing 1 out of $N=10$ actions.
Figure \ref{fig:CL_time_perception_accuracy} shows how the estimation error decreases over multiple episodes as the budget allocated to help the student increases.

\subsection{Other applications}
The two previous examples demonstrate the application of the framework when the observations are samples from a system of interest or actions performed by an agent. These settings extend to a large variety of problems, such as assisted language learning or improved hypothesis testing, and bring together a variety of fields such as input design and active learning. When training neural networks, for example, correcting the inputs could be compared to input design methods presented in Section \ref{ssec:relatedwork}. In reinforcement learning tasks, a teacher could use online correctional learning to accelerate the learning of the student in real time. 
The framework can also be easily translated to an adversarial setting, where the teacher finds the perturbation of the observations that maximizes the impact on the student's estimate -- e.g., data poisoning \cite[Section~6.1]{agrawal2019differentiable}. 




 
 



\section{Conclusion}
\label{sec:conclusions}

In this work we considered that an expert agent, a teacher, can observe the observations collected by a student agent from a certain system of interest. We used \textit{correctional learning} to study how the teacher can alter these observations in real time and under budget constraints in order to improve the learning process of the student. We bounded by how much the teacher can help the estimation of the student by reducing the variance of its estimate, and derived an MDP that gives the optimal policy to perform correctional learning in an online setting using dynamic programming. 
We illustrated the improvement of the estimation when using multinomial data (Figure \ref{fig:online_variances_budget}), and in a biological parameter estimation setting to illustrate the success of the framework in more complex settings (Figure \ref{fig:CL_time_perception_accuracy}).

The way is now paved for extending this method to several interesting applications mentioned throughout the paper, such as correctional reinforcement learning and comparison with related approaches. Tackling the dimensionality problem of MDPs will be an important step along the way.


\balance







\bibliographystyle{IEEEtran}
\bibliography{refs}             

\begin{thebibliography}{10}
\providecommand{\url}[1]{#1}
\csname url@samestyle\endcsname
\providecommand{\newblock}{\relax}
\providecommand{\bibinfo}[2]{#2}
\providecommand{\BIBentrySTDinterwordspacing}{\spaceskip=0pt\relax}
\providecommand{\BIBentryALTinterwordstretchfactor}{4}
\providecommand{\BIBentryALTinterwordspacing}{\spaceskip=\fontdimen2\font plus
\BIBentryALTinterwordstretchfactor\fontdimen3\font minus
  \fontdimen4\font\relax}
\providecommand{\BIBforeignlanguage}[2]{{%
\expandafter\ifx\csname l@#1\endcsname\relax
\typeout{** WARNING: IEEEtran.bst: No hyphenation pattern has been}%
\typeout{** loaded for the language `#1'. Using the pattern for}%
\typeout{** the default language instead.}%
\else
\language=\csname l@#1\endcsname
\fi
#2}}
\providecommand{\BIBdecl}{\relax}
\BIBdecl

\bibitem{cancer_detection}
L.~Shen, L.~R. Margolies, J.~Rothstein, E.~Fluder, R.~B. McBride, and W.~Sieh,
  ``Deep learning to improve breast cancer detection on screening
  mammography,'' \emph{Scientific Reports}, vol.~9, 2019.

\bibitem{cancer_prognosis}
K.~Kourou, T.~P. Exarchos, K.~P. Exarchos, M.~V. Karamouzis, and D.~I.
  Fotiadis, ``Machine learning applications in cancer prognosis and
  prediction,'' \emph{Computational and Structural Biotechnology Journal},
  vol.~13, pp. 8--17, 2015.

\bibitem{autonomous_driving}
A.~Geiger, P.~Lenz, and R.~Urtasun, ``Are we ready for autonomous driving? the
  kitti vision benchmark suite,'' in \emph{2012 IEEE Conference on Computer
  Vision and Pattern Recognition}, 2012, pp. 3354--3361.

\bibitem{oed:learn}
``learn, v.'' in \emph{OED Online}.\hskip 1em plus 0.5em minus 0.4em\relax
  Oxford University Press, Oct. 2021.

\bibitem{lfd}
B.~D. Argall, S.~Chernova, M.~Veloso, and B.~Browning, ``A survey of robot
  learning from demonstration,'' \emph{Robotics and Autonomous Systems},
  vol.~57, no.~5, pp. 469--483, 2009.

\bibitem{imitation_learning}
A.~Hussein, M.~Gaber, E.~Elyan, and C.~Jayne, ``Imitation learning: A survey of
  learning methods,'' \emph{ACM Computing Surveys}, vol.~50, 2017.

\bibitem{lourencco2021cooperative}
I.~Louren{\c{c}}o, R.~Mattila, C.~R. Rojas, and B.~Wahlberg, ``Cooperative
  system identification via correctional learning,'' \emph{19th IFAC Symposium
  on System Identification}, vol.~54, no.~7, pp. 19--24, 2021.

\bibitem{kira1992practical}
K.~Kira and L.~A. Rendell, ``A practical approach to feature selection,'' in
  \emph{Machine learning proceedings 1992}.\hskip 1em plus 0.5em minus
  0.4em\relax Elsevier, 1992, pp. 249--256.

\bibitem{wang2013online}
J.~Wang, P.~Zhao, S.~C. Hoi, and R.~Jin, ``Online feature selection and its
  applications,'' \emph{IEEE Transactions on knowledge and data engineering},
  vol.~26, no.~3, pp. 698--710, 2013.

\bibitem{pronzato2008optimal}
L.~Pronzato, ``Optimal experimental design and some related control problems,''
  \emph{Automatica}, vol.~44, no.~2, pp. 303--325, 2008.

\bibitem{hjalmarsson2009system}
H.~Hjalmarsson, ``System identification of complex and structured systems,''
  \emph{European journal of control}, vol.~15, pp. 275--310, 2009.

\bibitem{aggarwal2014active}
C.~C. Aggarwal, X.~Kong, Q.~Gu, J.~Han, and S.~Y. Philip, ``Active learning: A
  survey,'' in \emph{Data Classification}.\hskip 1em plus 0.5em minus
  0.4em\relax Chapman and Hall/CRC, 2014, pp. 599--634.

\bibitem{verma2020counterfactual}
S.~Verma, J.~Dickerson, and K.~Hines, ``Counterfactual explanations for machine
  learning: A review,'' \emph{arXiv preprint arXiv:2010.10596}, 2020.

\bibitem{kuusela2004learning}
P.~Kuusela and D.~Ocone, ``Learning with side information: {PAC} learning
  bounds,'' \emph{Journal of Computer and System Sciences}, vol.~68, no.~3, pp.
  521--545, 2004.

\bibitem{cover2006elements}
T.~M. Cover and J.~A. Thomas, \emph{Elements of Information Theory, second
  edition}.\hskip 1em plus 0.5em minus 0.4em\relax Wiley Interscience, 2006.

\bibitem{heirung2019input}
T.~A.~N. Heirung and A.~Mesbah, ``Input design for active fault diagnosis,''
  \emph{Annual Reviews in Control}, vol.~47, pp. 35--50, 2019.

\bibitem{chandola2009anomaly}
V.~Chandola, A.~Banerjee, and V.~Kumar, ``Anomaly detection: A survey,''
  \emph{ACM computing surveys}, vol.~41, no.~3, pp. 1--58, 2009.

\bibitem{krishnamurthy_2016}
V.~Krishnamurthy, \emph{Partially Observed Markov Decision Processes: From
  Filtering to Controlled Sensing}.\hskip 1em plus 0.5em minus 0.4em\relax
  Cambridge University Press, 2016.

\bibitem{puterman2014markov}
M.~L. Puterman, \emph{Markov Decision Processes: Discrete Stochastic Dynamic
  Programming}.\hskip 1em plus 0.5em minus 0.4em\relax John Wiley \& Sons,
  2014.

\bibitem{welniarz2021forward}
Q.~Welniarz, Y.~Worbe, and C.~Gallea, ``The forward model: a unifying theory
  for the role of the cerebellum in motor control and sense of agency,''
  \emph{Frontiers in Systems Neuroscience}, vol.~15, 2021.

\bibitem{lanillos2020review}
P.~Lanillos, D.~Oliva, A.~Philippsen, Y.~Yamashita, Y.~Nagai, and G.~Cheng, ``A
  review on neural network models of schizophrenia and autism spectrum
  disorder,'' \emph{Neural Networks}, vol. 122, pp. 338--363, 2020.

\bibitem{engl2009inverse}
H.~W. Engl, C.~Flamm, P.~K{\"u}gler, J.~Lu, S.~M{\"u}ller, and P.~Schuster,
  ``Inverse problems in systems biology,'' \emph{Inverse Problems}, vol.~25,
  no.~12, p. 123014, 2009.

\bibitem{lourencco2021biologically}
I.~Louren{\c{c}}o, R.~Mattila, R.~Ventura, and B.~Wahlberg, ``A
  biologically-inspired computational model of time perception,'' \emph{IEEE
  Transactions on Cognitive and Developmental Systems}, 2021.

\bibitem{lourencco2020teaching}
I.~Louren{\c{c}}o, R.~Ventura, and B.~Wahlberg, ``Teaching robots to perceive
  time: A twofold learning approach,'' in \emph{2020 Joint IEEE 10th
  International Conference on Development and Learning and Epigenetic Robotics
  (ICDL-EpiRob)}, 2020.

\bibitem{ludvig2008stimulus}
E.~A. Ludvig, R.~S. Sutton, and E.~J. Kehoe, ``Stimulus representation and the
  timing of reward-prediction errors in models of the dopamine system,''
  \emph{Neural computation}, vol.~20, no.~12, pp. 3034--3054, 2008.

\bibitem{agrawal2019differentiable}
A.~Agrawal, B.~Amos, S.~Barratt, S.~Boyd, S.~Diamond, and J.~Z. Kolter,
  ``Differentiable convex optimization layers,'' \emph{NEURIPS}, vol.~32, 2019.

\end{thebibliography}


\begin{appendices}
\section{Bounding the Decrease in Variance}
\label{app:bound}
    
Let us first present two known lemmas used to solve Theorem~\ref{thm:variance_bound}.
\begin{lemma}\label{lemma:lemma1}
    Let $X$ be a nonnegative random variable (R.V.). Then, 
    \begin{equation*}
        \mathbb{E}[X] = \int_0^\infty P(X \geq \tau)\dd \tau.
    \end{equation*}
\end{lemma}
\begin{proof}
Let $F$ be the CDF of $X$, i.e., $F(\tau) = P(X \leq \tau)$. Then, by integration by parts
\begin{equation*}
    \begin{aligned}
        \mathbb{E}[X] &= \int_0^\infty \tau \dd F(\tau) = -\int_0^\infty \tau \dd{[\underbrace{1-F(\tau)}_{=P(X > \tau)}]} \\
                      &=  -\tau[1-F(\tau)]\biggr\rvert_{\tau = 0}^\infty + \int_0^\infty \dd{[\underbrace{1-F(\tau)}_{=P(X > \tau)}]} \\
                      &= \int_0^\infty P(X > \tau)\dd \tau.
    \end{aligned}
\end{equation*}
Note that $P(X \geq \tau) = P(X > \tau) + P(X = \tau)$, where $P(X = \tau) > 0$ for at most a countable number of values of $\tau$, so
\begin{equation*}
    \int_0^\infty P(X = \tau)\dd \tau = 0
\end{equation*}
and
\begin{equation*}
    \mathbb{E}[X] = \int_0^\infty P(X \geq \tau) \dd \tau.
\end{equation*}
\end{proof}

\begin{lemma}\label{lemma:lemma2}
    Let $X$ be a R.V., and $\lambda$ a constant. Then, 
    \begin{equation*}
        \mathbb{E}[(X-\lambda)^2] = \int_0^\infty P(|X-\lambda|\geq \sqrt{\tau})\dd \tau.
    \end{equation*}
\end{lemma}
\begin{proof}
    Use Lemma \ref{lemma:lemma1} with $X$ replaced by $(X-\lambda)^2$. \\
\end{proof}

Let us now restate Theorem \ref{thm:variance_bound} and prove it in three parts.\\

\setcounter{theorem}{0}
\begin{theorem}
Let $X_1, \hdots, X_N$ be i.i.d. R.V.'s in $\{0,1,\hdots,M\}$ with mean $\mu$, and $Y = X_1 + \cdots + X_N$. Let $B\in\{0,1,\hdots,N\}$, and 
    \begin{equation*}
        \tilde{Y} = \underset{\{Z\in\{0,\hdots,N\} : |Y-Z| \leq B\}}{\arg\min} |Z-N\mu|.
    \end{equation*}
    Then, $\textnormal{var}[\tilde{Y}/N] \leq M^2\exp\left[-\frac{2B^2}{NM^2}\right]$. 
    Let us further assume that $X_i \sim \text{Unif} \left([0, \hdots, M]\right)$ (this assumption is not crucial but provides a special case that is easier to understand). Then, 
  \begin{equation*}
    \begin{aligned}
        \frac{\textnormal{var}[\tilde{Y}/N]}{\textnormal{var}[Y/N]} \leq \frac{6M}{5M+1}\exp\left[-\frac{2B^2}{NM^2}\right].\\
    \end{aligned}
\end{equation*}
\end{theorem}

\begin{proof}
Let us start by computing $\text{var}[\tilde{Y}]$. Using Hoeffding's inequality, we have that
\begin{equation*}
        \begin{aligned}
         P(|\tilde{Y}&-N\mu| \geq \sqrt{\tau})= \\
                 &= P(\tilde{Y} \geq N\mu + \sqrt{\tau}) + P(\tilde{Y} \leq N\mu - \sqrt{\tau}) \\
                 &= P(Y \geq N\mu + \sqrt{\tau} + B) + P(Y \leq N\mu - \sqrt{\tau}-B) \\
                 &\leq 2\exp\left[-\frac{2(\sqrt{\tau}+B)^2}{NM^2}\right].
        \end{aligned}
    \end{equation*}

    Therefore, from Lemma \ref{lemma:lemma2},
    \begin{equation*}
        \begin{aligned}
            &\text{var}[\tilde{Y}] = \mathbb{E}[(X-\lambda)^2]  \leq 2 \int_0^\infty \exp\left[-\frac{2(\sqrt{\tau}+B)^2}{NM^2}\right] \dd \tau \\
                           &= 2  \int_0^\infty \exp\left[-\frac{2 u^2}{NM^2}\right]\cdot 2(u-B) \dd{u} \\
                           &= 4 \int_0^\infty u\exp\left[-\frac{2 u^2}{NM^2}\right]\dd u - 4B\int_0^\infty \exp\left[-\frac{2 u^2}{NM^2}\right]\dd u \\
                           &\leq 4 \int_{\frac{2B^2}{NM^2}}^\infty \frac{NM^2}{4}e^{-v}\dd{v} = NM^2\exp\left[-\frac{2B^2}{NM^2}\right].
        \end{aligned}
    \end{equation*}

Let us now compute $\text{var}[Y]$. If $X_i \sim \textnormal{Unif}([0,\dots,M]),$
    \begin{equation*}
    \begin{aligned}
        &\text{var}[Y] = \sum_{k=0}^M\left[k-\frac{M}{2}\right]^2\frac{1}{M+1}  \\
      & = \frac{1}{M+1}\sum_{k=0}^M[k^2-kM+M^2] \\
              &        =\frac{1}{M+1} [M^2(M+1)-
                      M\frac{M(M+1)}{2}+  \\ 
                     & \qquad \qquad \qquad \qquad \qquad + \frac{1}{6}M(M+1)(2M+1)] \\
              &        = M^2-\frac{M^2}{2} +\frac{1}{6}M(2M+1)    = \frac{5M^2+M}{6}.
    \end{aligned}
    \end{equation*}
Finally, we conclude that
\begin{equation*}
    \begin{aligned}
        \frac{\text{var}[\tilde{Y}/N]}{\text{var}[Y/N]} \leq \frac{M^2\exp\left[-\frac{2B^2}{NM^2}\right]}{\frac{5M^2+M}{6}} = \frac{6M}{5M+1}\exp\left[-\frac{2B^2}{NM^2}\right].
    \end{aligned}
\end{equation*}

\end{proof}


\section{Results for binomial data}
\label{app:binomial}

We exemplify now the case in which the observations are randomly sampled from a binomial distribution with parameter $\theta_0=0.5$, for comparison with the batch case from \cite{lourencco2021cooperative}). We define $N=10$ and repeat the algorithm over 50 experiments.

Figure \ref{fig:online_error_binomial} shows the results of the MDP proposed in Section \ref{sec:online} for performing online correctional learning now for a binomial distribution instead of a multinomial as in Section~\ref{sec:results}. A main difference in this case is that we can prove that the policy learned by the teacher is optimal, as stated in Section \ref{sec:online}, since the error of the estimate of the student with the corrected sequence of observations (in orange) -- that is, without the help of the teacher -- is always exactly the minimum error attainable for the original sequence (in black). The exact expression for this error is: 
\begin{equation}
    \begin{aligned}
        e(N,\theta_0,b,\hat{\theta}) = \max\left\{ \| \theta_0 - \hat{\theta} \|_1 - \frac{2 b}{N} , e_{min} \right\}.
        \label{eq:error}
    \end{aligned}
\end{equation}
where the second term is the $e_{min}$ from \eqref{eq:emin}. This error is now always never larger than the error with the original sequence of observations (in blue).

Here, an analysis of the policy values obtained using the dynamic programming algorithm shows that indeed the teacher chooses the optimal policy, defined in \eqref{eq:optpolicy}, of delaying spending its budget as much as possible and doing so only once too many of one of the outcomes is sampled. In this case switching to the other value does not represent a risk, unlike in the multinomial case where a specific outcome to change to has to be chosen. 
A specific example from Figure \ref{fig:online_error_binomial} is the sequence $11011100{\color{orange}\textbf{0}}1$, where the teacher alters the orange value that was a 1 to a 0 regardless the value of the next sample.

\pgfplotstableread{
1.  2.  3.  4.  5.  6.  7.  8.  9. 10. 11. 12. 13. 14. 15. 16. 17. 18. 19. 20. 21. 22. 23. 24. 25. 26. 27. 28. 29. 30. 31. 32. 33. 34. 35. 36. 37. 38. 39. 40. 41. 42. 43. 44. 45. 46. 47. 48. 49. 50.
0.2 0.2 0.3 0.2 0.  0.  0.1 0.  0.2 0.2 0.1 0.2 0.2 0.  0.  0.1 0.2 0.2 0.1 0.2 0.1 0.2 0.1 0.  0.1 0.1 0.2 0.1 0.2 0.2 0.  0.  0.1 0.  0.  0. 0.  0.  0.1 0.3 0.1 0.4 0.2 0.4 0.1 0.1 0.1 0.1 0.1 0.2
0.1 0.1 0.2 0.1 0.  0.  0.  0.  0.1 0.1 0.  0.1 0.1 0.  0.  0.  0.1 0.1 0.  0.1 0.  0.1 0.  0.  0.  0.  0.1 0.  0.1 0.1 0.  0.  0.  0.  0.  0. 0.  0.  0.  0.2 0.  0.3 0.1 0.3 0.  0.  0.  0.  0.  0.1
0.1 0.1 0.2 0.1 0.  0.  0.  0.  0.1 0.1 0.  0.1 0.1 0.  0.  0.  0.1 0.1 0.  0.1 0.  0.1 0.  0.  0.  0.  0.1 0.  0.1 0.1 0.  0.  0.  0.  0.  0. 0.  0.  0.  0.2 0.  0.3 0.1 0.3 0.  0.  0.  0.  0.  0.1
}\datasetinitB
\pgfplotstabletranspose\datasetbinomial{\datasetinitB}

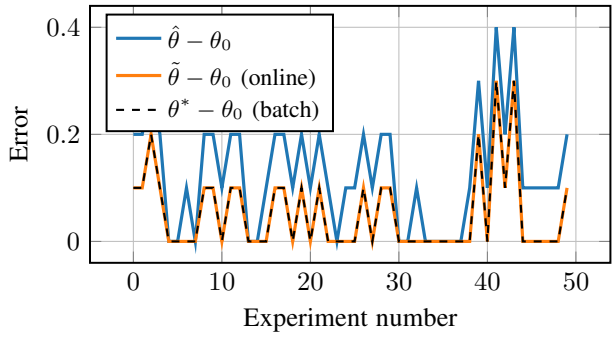
\begin{figure}[tb]
\centering
\begin{tikzpicture}
  \begin{axis}[
height=5cm,width=8.5cm,ylabel={Error}, xlabel={Experiment number}, grid, legend style={font=\small}, legend cell align={left}, legend pos = {north west},
		]
    \addplot[color=blue1,very thick]  table[x index=0, y index=2] \datasetbinomial; 
    \addplot[color=orange1,very thick] table[x index=0,y index=3] \datasetbinomial; 
    \addplot[color=black,thick, dashed] table[x index=0,y index=4] \datasetbinomial; 
    \legend{$\hat{\theta}-\theta_0$,$\tilde{\theta}-\theta_0$ (online),$\theta^*-\theta_0$ (batch)}
  \end{axis}
\end{tikzpicture}
    \caption{Error of the estimate obtained online by the student with (in orange), and without (in blue) the help of the teacher. The former coincides with the black line, which is the theoretical best estimate for that situation \eqref{eq:error} ($b=1,N=10$). The bigger the $b$, the closer the orange and black lines become to the minimum error $e_{min}$ \eqref{eq:emin}.}
    \label{fig:online_error_binomial}
\end{figure}


\end{appendices}

\end{document}